\documentclass[10pt,journal]{IEEEtran}
\usepackage{amssymb}
\usepackage{amsmath,amsfonts}
\usepackage{algorithmic}
\usepackage{algorithm}
\usepackage{array}
\usepackage[caption=false,font=normalsize,labelfont=sf,textfont=sf]{subfig}
\usepackage{textcomp}
\usepackage{stfloats}
\usepackage{url}
\usepackage{verbatim}
\usepackage{graphicx}
\usepackage{cite}
\usepackage{amsthm}
\usepackage{booktabs} 
\usepackage{multirow}
\usepackage{array}
\usepackage{enumitem}
\usepackage{colortbl}
\usepackage{makecell}
\usepackage{threeparttable}
\newtheorem{theorem}{Theorem}

\usepackage{xcolor}

\def\ie{{\textit{i}.\textit{e}.}}
\def\etal{{{et al}.}}
\usepackage{marvosym}
\usepackage{caption}

\usepackage{hyperref}
\hypersetup{hidelinks} %

\begin{document}

\title{SEGA: A Transferable \emph{S}igned \emph{E}nsemble \emph{G}aussian Black-Box \emph{A}ttack against No-Reference Image Quality Assessment Models}

\author{Yujia Liu, Dingquan Li \IEEEmembership{Member, IEEE}, Zhixuan Li, 
        and 
        Tiejun Huang\textsuperscript{\Letter}, \IEEEmembership{Senior Member, IEEE}%
\thanks{Yujia Liu, and Tiejun Huang are with the School of Computer Science, the National Key Laboratory for Multimedia Information Processing, National Engineering Research Center of Visual Technology, Peking University, Beijing, China (e-mail: yujia\_liu@pku.edu.cn, tjhuang@pku.edu.cn). \emph{(Corresponding author:
Tiejun Huang.)}}
\thanks{Zhixuan Li is with the College of Computing and Data Science, Nanyang Technological University, Singapore (e-mail: zhixuanli520@gmail.com)}
\thanks{Dingquan Li is with the Department of Networked Intelligence, Pengcheng Laboratory, Shenzhen, China (e-mail: dingquanli@pku.edu.cn).}
}

\markboth{Preprint}%
{Yujia Liu \MakeLowercase{\textit{et al.}}: SEGA: A Transferable \underline{S}igned \underline{E}nsemble \underline{G}aussian Black-Box \underline{A}ttack against No-Reference Image Quality Assessment Models}

\maketitle

\begin{abstract}
No-Reference Image Quality Assessment (NR-IQA) models play an important role in various real-world applications. Recently, adversarial attacks against NR-IQA models have attracted increasing attention, as they provide valuable insights for revealing model vulnerabilities and guiding robust system design. Some effective attacks have been proposed against NR-IQA models in white-box settings, where the attacker has full access to the target model. However, these attacks often suffer from poor transferability to unknown target models in more realistic black-box scenarios, where the target model is inaccessible. This work makes the first attempt to address the challenge of low transferability in attacking NR-IQA models by proposing a transferable Signed Ensemble Gaussian black-box Attack (SEGA). The main idea is to approximate the gradient of the target model by applying Gaussian smoothing to source models and ensembling their smoothed gradients. To ensure the imperceptibility of adversarial perturbations, SEGA further removes inappropriate perturbations using a specially designed perturbation filter mask. Experimental results demonstrate the superior transferability of SEGA, validating its effectiveness in enabling successful transfer-based black-box attacks against NR-IQA models. Code for this paper is available at \url{https://github.com/YogaLYJ/SEGA_IQA}.
\end{abstract}

\section{Introduction}
Image Quality Assessment (IQA) models, which predict quality scores for input images, play an important role in various downstream tasks~\cite{2019_IEEE_QA_recommendation,face_IQA,brain_inspired_QA}. For instance, the predicted scores can serve as a key factor in recommendation systems~\cite{2019_IEEE_QA_recommendation}. In addition, they are often used to guide model optimization and performance evaluation in domains such as medical diagnosis~\cite{2023_medical} and autonomous driving~\cite{2023_autodrive,driving_QA}.
Based on the availability of reference images, IQA can be categorized into Full-Reference IQA (FR-IQA) and No-Reference IQA (NR-IQA). FR-IQA models assess the quality of distorted images with corresponding reference images~\cite{2018_CVPR_Richard_LPIPS,2020_TPAMI_Ding_DISTS}, whereas NR-IQA models predict quality scores without any reference~\cite{2022_CVPRw_MANIQA,2020_TCSVT_DBCNN}. This paper focuses on NR-IQA, which poses greater challenges and is commonly used in real-world applications~\cite{NR_challenging}.

With the widespread deployment of NR-IQA models in various downstream applications, adversarial attacks targeting these models have attracted increasing attention~\cite{2022_NIPS_Zhang_PAttack,2022_BMVC_Ekarerina_UAP,2024_TCSVT_SurFreeIQA}. Such attacks aim to mislead predicted quality scores by introducing imperceptible perturbations to input images, thereby exposing potential vulnerabilities in NR-IQA models that may pose serious security risks in real-world scenarios. For example, if an NR-IQA model assigns a high-quality score to a low-quality image under attack, it may mislead recommendation systems into promoting poor-quality content, resulting in a degraded user experience~\cite{2022_WACV_NR_application,2016_HVEI_QA_application}. In more critical domains such as healthcare, the misjudgment of low-quality medical images could severely compromise diagnostic outcomes~\cite{2024_America_QA_medical,medical_QA}. 
Therefore, investigating adversarial attacks against NR-IQA models is essential for identifying their weaknesses and guiding the development of more robust and trustworthy systems.

\begin{figure*}[!t]
    \centering
    \includegraphics[width=0.98\linewidth]{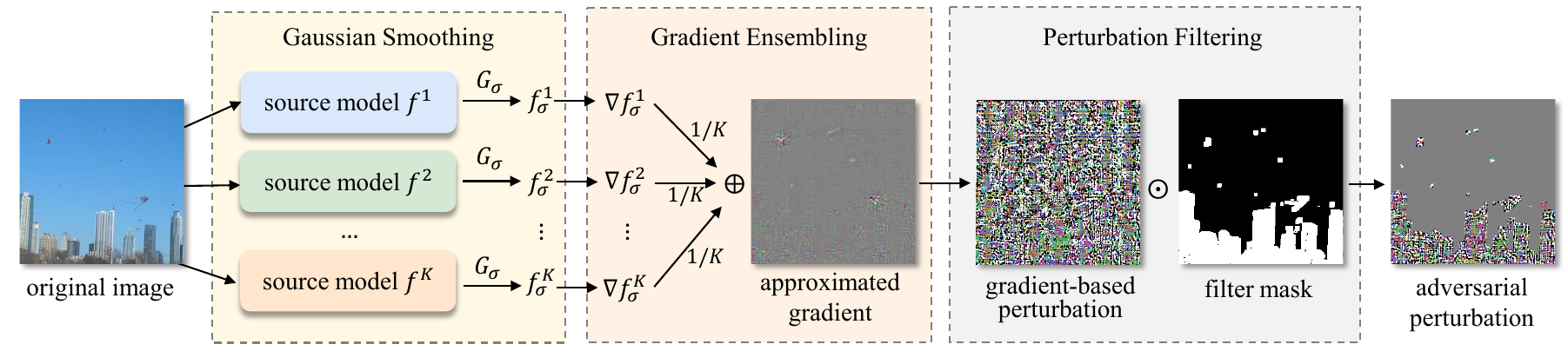}
    \caption{Overview of the proposed SEGA method. SEGA leverages ensembled gradients from multiple source models with Gaussian smoothing to approximate the gradient of the target model. To enhance the imperceptibility of adversarial perturbations, SEGA designs a perturbation filter mask to remove inappropriate perturbations.}
    \label{fig:intro_pipline}
\end{figure*}

Recent studies on attacking NR-IQA models have primarily focused on white-box scenarios, which assume full access to the target model. In such settings, as demonstrated in other tasks, adversarial example generation heavily relies on the input image's gradient~\cite{Segmentation_attack,ranking_attack}, since it indicates the direction that most rapidly alters prediction scores. As for the NR-IQA task, Shumitskaya~\etal~\cite{2024_CVIU_Ekaterina_OUAP} proposed a universal perturbation by aggregating the gradients across the entire test set with weighted summation, effectively increasing the predicted scores for most inputs. In contrast, Zhang~\etal~\cite{2022_NIPS_Zhang_PAttack} formulated an optimization problem for individual images with Lagrange multipliers and solved it using gradient descent methods, achieving significant errors in prediction consistency across a dataset.

However, research on attacking NR-IQA models in the more challenging black-box scenario remains limited. In this setting, adversarial perturbations are generated solely based on input-output pairs, without access to the model’s internal information~\cite{transfer_Papernot,black_box_benchmark}. Yang~\etal~\cite{2024_TCSVT_SurFreeIQA} proposed a query-based method that iteratively queries the target model and performs a binary search in the polar space to locate effective perturbations. While effective at inducing prediction errors, this approach is highly time-consuming due to the need for thousands of queries. 
An alternative strategy is transfer-based black-box attacks, where attackers generate adversarial examples using accessible source models and then transfer them to the target model~\cite{transfer_attack}. However, existing attacks on NR-IQA models have demonstrated limited transferability to unknown target models~\cite{2022_NIPS_Zhang_PAttack,2024_CVPR_Liu_NT}, and to date, no attack method with good transferability for NR-IQA tasks has been explored.

In this paper, we propose a black-box attack method against NR-IQA models that achieves strong transferability. The main idea is to approximate the gradients of the target model using accessible source models while maintaining the imperceptibility of adversarial perturbations. Existing white-box attacks~\cite{2022_NIPS_Zhang_PAttack,2024_CVIU_Ekaterina_OUAP} have highlighted the importance of gradients, which reveal the direction of the steepest change in prediction scores. Given that different NR-IQA models are designed to solve the same task, it is reasonable to assume that they share similar gradient behaviors. However, previous studies have shown that raw gradients often contain significant noise~\cite{2017_arxiv_Smilkov_SmoothGrad,2017_ICML_Balduzzi_NoisyGrad}, which may lead to inconsistencies between the gradients of the source and target models, and thereby reduce the transferability of adversarial examples. To address this, we aim to remove noise in the source model gradients to better approximate the target model’s gradient directions. 

In detail, we propose the Signed Ensemble Gaussian Attack (SEGA) method on NR-IQA models, which consists of three steps: Gaussian smoothing, gradient ensembling, and perturbation filtering (as shown in Fig.~\ref{fig:intro_pipline}). The first two steps aim to improve gradient approximation, while the final step enhances the imperceptibility of the resulting perturbations. Given an input image and a source model, we first apply Gaussian smoothing to the model and use the gradient of the smoothed output as an approximation. This is because smoothed gradients can effectively reduce noise in the original gradient~\cite{2017_arxiv_Smilkov_SmoothGrad}. Next, inspired by advances in transfer-based attacks for classification models~\cite{2018_CVPR_Dong_MI-FGSM}, we employ an ensemble of smoothed source models to further refine the gradient approximation and improve transferability. The adversarial perturbation is then generated by taking the sign of the approximated gradient scaled by a small step size. Finally, to improve the imperceptibility of adversarial perturbations, we design a perturbation filter mask that removes inappropriate perturbations while maintaining the attack effectiveness.

The main contributions of this work are as follows:
\begin{enumerate}
    \item Methodologically, we propose SEGA—the first transfer-based black-box attack method specifically designed for NR-IQA models. SEGA introduces Gaussian smoothing into the adversarial attack framework for the first time, enabling a more accurate approximation of target model gradients.
    \item Theoretically, we analyze the upper bound on the error between the approximated gradient obtained from source models and the true gradient of the target model, providing insights into the effectiveness of gradient smoothing and ensembling.
    \item Empirically, extensive experiments on the CLIVE dataset~\cite{2015_TIP_LIVEC} and KonIQ-10k dataset~\cite{koniq_10k} demonstrate that SEGA significantly enhances the transferability of adversarial examples across NR-IQA models, as evaluated by both prediction accuracy and score consistency.
\end{enumerate}

\section{Related Works}

\subsection{NR-IQA Models}
The task of NR-IQA is to predict human-perceived quality, typically measured by the Mean Opinion Score (MOS). This is achieved by analyzing a distorted image directly, without any reference images~\cite{NR_2012,Weixia_LwFKG,2023_Cao_NRIQA}.

Traditional NR-IQA metrics rely on manually designed features~\cite{ghadiyaram2017perceptual} intended to capture the characteristics of the Human Visual System (HVS). This approach is motivated by the assumption that distortions in the statistical regularities of natural images~\cite{simoncelli2001natural} often serve as a reasonable proxy for perceived visual quality. For example, NIQE~\cite{2012_SP_NIQE} and BRISQUE~\cite{2012_TIP_BRISQUE} are widely used metrics that assess image quality using natural scene statistics (NSS) in the spatial domain, while other methods extract NSS features from the transform domain~\cite{moorthy2011blind,saad2012blind}.

The advancement of neural networks has popularized deep learning-based approaches in NR-IQA, effectively addressing the challenges of large-scale datasets~\cite{Wang_gMAD,ying2020patches,zhu2020metaiqa} and limited human labels~\cite{Bosse_NR,Liu_RankIQA,Zhang_BIQA}. Research in this direction has explored both convolutional neural networks and transformer architectures~\cite{2020_CVPR_hyperIQA,2020_MM_LinearityIQA,Ke_2021_ICCV_MUSIQ}. For instance, DBCNN~\cite{2020_TCSVT_DBCNN} employs dual subnetworks to separately extract distortion-related and semantic features. In contrast, MANIQA~\cite{2022_CVPRw_MANIQA} utilizes the vision transformer~\cite{dosovitskiy2020vit} to enhance features through cross-channel interactions. More recently, Zhang~\etal~\cite{2023_CVPR_LIQE} used the CLIP model~\cite{radford2021learning} within a multitask learning framework, introducing auxiliary knowledge to boost BIQA performance.

\subsection{White-Box Attacks Against NR-IQA Models}
Several studies have explored white-box attacks on NR-IQA models and achieved impressive attack performance with the use of input image gradients. 

Several studies have directly adapted classical gradient-based adversarial attacks from image classification to the NR-IQA task. For example, Meftah~\etal~\cite{meftah2023evaluating} applied the well-known FGSM~\cite{2015_ICLR_Goodfellow_FGSM} and PGD~\cite{2018_ICLR_Madry_PGD} attacks to NR-IQA models, while Gushchin~\etal~\cite{gushchin2024adversarial} further explored variants of FGSM~\cite{2018_CVPR_Dong_MI-FGSM,2023_Sang_AMI_FGSM} in this context. 

Other studies have developed adversarial attacks specifically targeting NR-IQA tasks. Zhang~\etal~\cite{2022_NIPS_Zhang_PAttack} proposed the perceptual attack by formulating it as an optimization problem with Lagrange multipliers, which they solved using the gradient descent method. This approach significantly reduced prediction accuracy and caused prediction errors. Shumitskaya~\etal~\cite{2022_BMVC_Ekarerina_UAP} proposed a universal perturbation method and its variants~\cite{2024_CVIU_Ekaterina_OUAP} that generate adversarial perturbations by computing a weighted sum of image gradients across the test dataset. The resulting perturbation can effectively increase the predicted scores for most inputs.
Recently, some attacks~\cite{2024_ICML_Ekaterina_IOI,abud2025evaluating} were proposed to process the adversarial perturbations in the frequency domain to improve the imperceptibility of perturbations. These works highlight the important role of gradients in attack methods and demonstrate the vulnerability of NR-IQA models to white-box adversarial attacks.

\subsection{Black-Box Attacks Against NR-IQA Models}
Black-box attacks on NR-IQA models have been relatively underexplored. There are two main strategies for attacking models in the black-box scenario. One is the query-based approach, where attackers can query the target model multiple times to approximate the gradient, which is then used to optimize the perturbations and generate adversarial examples. Yang~\etal~\cite{2024_TCSVT_SurFreeIQA} adapted the approach from SurFree~\cite{2021_CVPR_Maho_SurFree}, employing binary search in polar coordinates to generate adversarial examples by continuously querying the changes in predicted scores. Ran~\etal~\cite{2025_ESWA_RandomQuery} developed a query-based attack based on a random search paradigm. Although query-based attacks can achieve impressive attack performance in black-box scenarios, they require a large number of queries (nearly 10,000 times~\cite{2024_TCSVT_SurFreeIQA,2025_ESWA_RandomQuery}) to generate a single adversarial example, making the process highly time-consuming.

Compared to query-based attacks, transfer-based approaches are more time-efficient. Transfer-based attacks have been widely studied in the context of classification model attacks~\cite{2019_CVPR_YiranChen_AA,2022_CompSec_Liu_LowFreTransAdv}, but remain largely unexplored for NR-IQA models. Existing adversarial attacks on NR-IQA often suffer from poor transferability across different models~\cite{2022_NIPS_Zhang_PAttack,2025_ESWA_RandomQuery}. Korhonen~\etal~\cite{2022_QEVMAw_Korhonen_BIQA} attributed this issue to the diverse features learned by different NR-IQA models. To address this, they re-trained a ResNet-based source model~\cite{2016_CVPR_He_ResNet} to capture more generic quality assessment features. While adversarial examples generated from this model get some improvements in transferability, the gains were limited. Moreover, re-training the source model causes considerable computational overhead.

\section{Method}
In this paper, we propose a transfer-based black-box attack method against NR-IQA models, named Signed Ensemble Gaussian Attack (SEGA). SEGA consists of three components: Gaussian smoothing (Sec.~\ref{sec:gaussian_smoothing}), gradient ensembling (Sec.~\ref{sec:gradient_ensemble}), and perturbation filtering (Sec.~\ref{sec:pert_filter}), as illustrated in Fig.~\ref{fig:intro_pipline}. The first two parts aim to reduce noise in the gradients and accurately approximate the gradients of the target model. Perturbation filtering focuses on enhancing the imperceptibility of perturbations to ensure that adversarial examples are perceptually indistinguishable from the original images.

\subsection{Problem Definition}
In the transfer-based black-box scenario, there is an unknown NR-IQA target model $h$ and an available source model $f$. Given an input image $x$, the objective is to design an adversarial perturbation based on the source model $f$, denoted as $\Delta x (f)$, that maximally manipulates the prediction score of $x + \Delta x(f)$ by $h$. Meanwhile, the adversarial image $x + \Delta x(f)$ should maintain perceptual similarity to the original image $x$, ensuring that both images receive the same objective score from human observers,~\ie
\begin{equation}
\label{eq:target}
\max_{\Delta x(f)}~|h(x + \Delta x(f)) - h(x)| \quad \text{s.t.}~D(x, x + \Delta x(f)) \leqslant \epsilon,
\end{equation}
where $D(\cdot, \cdot)$ measures the perceptual distance between the original image and the adversarial example, $\epsilon$ represents a predefined threshold.
We further specify the attack direction as follows. Assuming that the predicted quality scores are normalized to the range $[0,100]$, the attacker aims to decrease the predicted score when $h(x)>50$; otherwise, the objective is to increase the predicted score of the adversarial image.
In this paper, we define $D(x, x + \Delta x(f)) = \Vert \Delta x(f)\Vert_\infty$ with the $\ell_\infty$ function commonly used in previous studies~\cite{2015_ICLR_Goodfellow_FGSM,2024_CVPR_Liu_NT,2024_ICML_Ekaterina_IOI}. In the following text, we use $\Delta x$ instead of $\Delta x(f)$ for its simplicity. 

Since the input gradient of the target model is inaccessible, the perturbation $\Delta x$ is crafted by solving the following white-box optimization problem on the accessible source model $f$:
\begin{equation}
\label{eq:single_source}
\max_{\Delta x}~|f(x + \Delta x) - f(x)| \quad \text{s.t.}~\Vert \Delta x\Vert_\infty \leqslant \epsilon.
\end{equation}
This white-box problem has been widely explored~\cite{2022_NIPS_Zhang_PAttack,2024_CVIU_Ekaterina_OUAP,2015_ICLR_Goodfellow_FGSM}. 
Due to the computational efficiency, we adopt the idea from one-step FGSM~\cite{2015_ICLR_Goodfellow_FGSM} attack to solve Eq.~\eqref{eq:single_source},~\ie

\begin{equation}
\label{eq:FGSM}
\Delta x =
\begin{cases}
-\epsilon \cdot \text{sgn}(\nabla f(x)), & \text{if } h(x)>50, \\
\epsilon \cdot \text{sgn}(\nabla f(x)), & \text{otherwise},
\end{cases}
\end{equation}
where $\text{sgn}(\cdot)$ denotes the sign function, and $\nabla f(x)$ is the gradient of the source model $f$ with respect to the input image $x$.

\subsection{Gaussian Smoothing}
\label{sec:gaussian_smoothing}

According to Eq.~\eqref{eq:FGSM}, the transferability of $x + \Delta x$ from the source model $f$ to the target model $h$ primarily depends on the extent of $h$'s instability along the direction of $\text{sgn}(\nabla f(x))$. While the level sets of models within the same task are expected to exhibit global similarity~\cite{2017_arxiv_Tram_TransSpace,2018_arxiv_Wu_TransSmooth}, local fluctuations caused by noise~\cite{2017_arxiv_Smilkov_SmoothGrad} in $\nabla f(x)$ may hurt this similarity. An efficient way to remove noise in $\nabla f(x)$ is smoothing $f$ with a Gaussian mollifier~\cite{2017_arxiv_Smilkov_SmoothGrad}, \ie
\begin{equation}
\begin{split}
    \label{eq:f_sigma_continous}
    f_\sigma (x) &= \frac{1}{(2\pi)^{d/2}} \int_{\mathbb{R}^d}
                        f(x+\sigma u) e^{-\Vert u \Vert^2_2~/2} du \\
                 &= \mathbb{E}_{u\sim \mathcal{N}(0,I_d)} [f(x+\sigma u)].
\end{split}
\end{equation}
In this formula, $d$ denotes the dimension of the image $x$, $\mathcal{N}(0, I_d)$ is the $d$-dimensional standard normal distribution with $I_d$ as the identity matrix, and $\sigma > 0$ determines the standard deviation of the term $\sigma u$. Additionally, $\sigma$ controls the degree of smoothing. Larger $\sigma$ values result in a smoother approximation, while smaller $\sigma$ values preserve more details of the model $f$~\cite{2023_arxiv_GSmoothGD}. This trade-off is formalized in Theorem~\ref{thr:Gaussian_sigma} and visually illustrated in Fig.~\ref{fig:Gaussian_smooth}.

\begin{figure}
    \centering
    \includegraphics[width=0.8\linewidth]{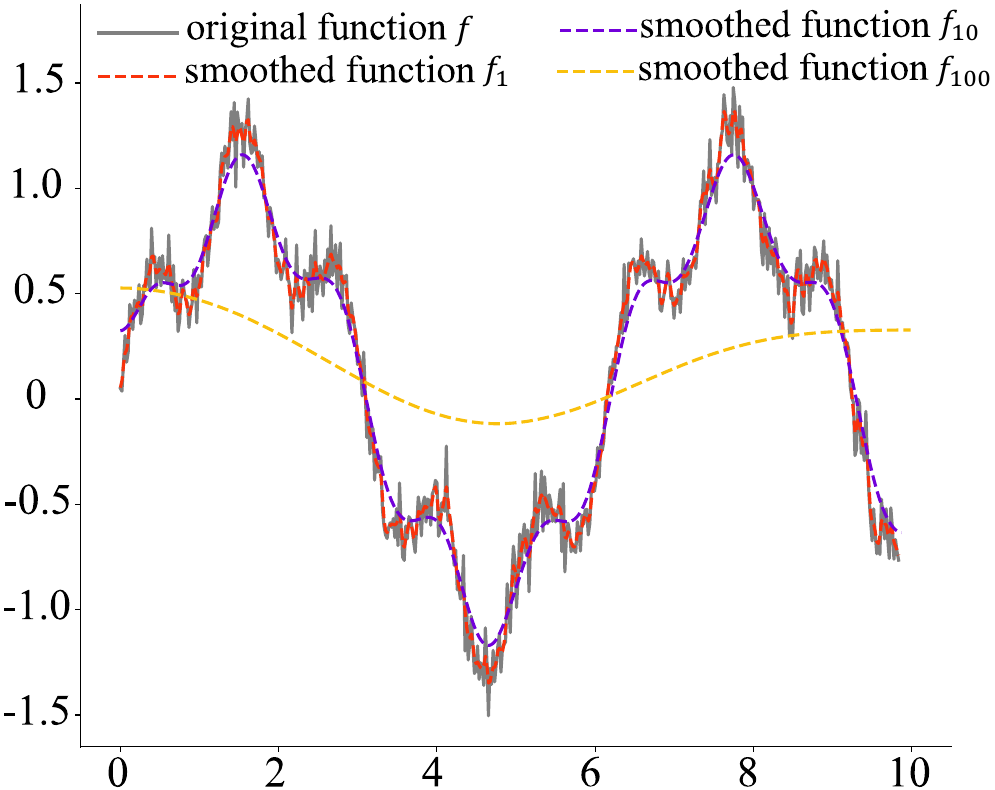}
    \caption{An illustration of Gaussian smoothing is presented. The original function is defined as $f(x) = \sin(x) + 0.3 \cdot \sin(5x) + \eta$, where $\eta$ represents random noise drawn from a normal distribution $\mathcal{N}(0,0.1)$. Gaussian smoothing is applied with three different standard deviation parameters: $\sigma = {1, 10, 100}$. The results demonstrate that a small $\sigma$ preserves excessive detail from the original noise, while an excessively large $\sigma$ leads to oversmoothing and thereby loses important signal characteristics. Therefore, selecting an appropriate $\sigma$ value is crucial for effective Gaussian smoothing.}
    \label{fig:Gaussian_smooth}
\end{figure}

\begin{theorem}
\label{thr:Gaussian_sigma}
    Suppose $f$ is Lipschitz-continuous, and $f_\sigma$ is the Gaussian smoothing of $f$. Then, for every $x$, we have
    \begin{equation}
        \lim_{\sigma\rightarrow 0} f_\sigma(x) = f(x).
    \end{equation}
\end{theorem}

\begin{proof}
    Suppose $f$ is Lipschitz-continuous with $L$ as the Lipschitz constant.
    Recall that 
    \begin{equation}
        f_\sigma (x) = \frac{1}{(2\pi)^{d/2}} \int_{\mathbb{R}^d}
                        f(x+\sigma u) e^{-\Vert u \Vert^2_2~/2} du,
    \end{equation}
    so
    \begin{equation}
    \begin{split}
        & ~~~~~~~~~~~~~~~~~~~~~|f_\sigma (x) - f(x)|  \\
        &= \left|
        \frac{1}{(2\pi)^{d/2}} \int_{\mathbb{R}^d}
                        \left(f(x+\sigma u) - f(x)\right) e^{-\Vert u \Vert^2_2~/2} du
                        \right| \\
            &\leqslant \frac{1}{(2\pi)^{d/2}} \int_{\mathbb{R}^d}
                        |f(x+\sigma u)-f(x)| e^{-\Vert u \Vert^2_2~/2} du.
    \end{split}
    \end{equation}
    Since $f$ is Lipschitz-continuous, we have
    \begin{equation}
        |f(x+\sigma u)-f(x)| \leqslant L \Vert \sigma u\Vert.
    \end{equation}
    Therefore, 
    \begin{equation}
    \begin{split}
        |f_\sigma (x) - f(x)|
            &\leqslant \frac{L\cdot \sigma}{(2\pi)^{d/2}} \int_{\mathbb{R}^d}
                        \Vert u\Vert e^{-\Vert u \Vert^2_2~/2} du \\
            & = (L\cdot \sigma) \mathbb{E}_{u\sim \mathcal{N}(0,I_d)}[\Vert u \Vert].
    \end{split}
    \end{equation}
    Since $u_i\sim \mathcal{N}(0,1)$, we have $u_i^2\sim \chi^2(1)$. Therefore, $\sum_{i=1}^d u_i^2\sim \chi^2(d)$,~\ie 
    \begin{equation}
        \Vert u \Vert_2^2 \sim \chi^2 (d).
    \end{equation}
    Let the variables be defined as $Y = \Vert u \Vert_2^2$ and $V = \Vert u \Vert$. Using the transformation technique with $V = q(Y) = \sqrt{Y}$, the probability density function of $V$ is given by
    \begin{equation}
        p_V(v) = p_Y(q^{-1}(y)) \left|
            \frac{dy}{dv}
            \right| = \frac{y^{d-1}e^{-y^2/2}}{2^{d/2-1}\Gamma(\frac{d}{2})},~~y>0,
    \end{equation}
    where $\Gamma(\cdot)$ represents the Gamma function~\cite{2002_gamma}.
    Therefore, the expectation of $V=\Vert u \Vert$ is
    \begin{equation}
        \mathbb{E}_{u\sim \mathcal{N}(0,I_d)}[\Vert u \Vert] = 
        \int_0^\infty  \frac{y^d e^{-y^2/2}}{2^{d/2-1}\Gamma(\frac{d}{2})} dy= \sqrt{2} \frac{\Gamma(\frac{d+1}{2})}{\Gamma(\frac{d}{2})}.
    \end{equation}
    Therefore,
    \begin{equation}
    \begin{split}
        \label{eq:ineq_fsigma_f}
        |f_\sigma (x) - f(x)| & \leqslant (L\cdot \sigma) \mathbb{E}_{u\sim \mathcal{N}(0,I_d)}[\Vert u \Vert] \\
        & = (L\cdot \sigma) \sqrt{2} \frac{\Gamma(\frac{d+1}{2})}{\Gamma(\frac{d}{2})}.
    \end{split}
    \end{equation}
    Let $\sigma\rightarrow 0$ in Eq.~\eqref{eq:ineq_fsigma_f}, we have
    \begin{equation}
        \lim_{\sigma\rightarrow 0} |f_\sigma (x) - f(x)| = 0.
    \end{equation}

    Thus, we have proved
    \begin{equation}
        \lim_{\sigma\rightarrow 0} f_\sigma(x) = f(x).
    \end{equation}
\end{proof}

According to Eq.~\eqref{eq:f_sigma_continous}, we define
\begin{equation}
f_\sigma(x) = \frac{1}{m} \sum_{i=1}^m f(x + \sigma u_i), \quad\text{where } u_i \sim \mathcal{N}(0, I_d),
\end{equation}
as the discrete Gaussian-smoothed version of $f$, where $m$ is the number of samples. The gradient $\nabla f(x)$ is then substituted by $\nabla f_\sigma(x)$,~\ie
\begin{equation}
\label{eq:smooth_gradient}
    \nabla f(x) \approx \nabla f_\sigma(x) = \frac{1}{m}\sum_{i=1}^m \nabla f(x+\sigma u_i).
\end{equation}

\subsection{Gradient Ensembling}
\label{sec:gradient_ensemble}
Ensembling source models to generate more transferable adversarial examples is a commonly used strategy in image classification tasks~\cite{2018_CVPR_Dong_MI-FGSM,2020_ICLR_Lin_NI-FGSM,2022_CompSec_Liu_LowFreTransAdv}. This approach is often considered a form of model augmentation, where multiple source models are attacked simultaneously. 

Inspired by this idea, we ensemble the gradients of $K$ source models $\{f^1,\dots,f^K\}$ to further reduce the noise present in the gradient of any single source model. 
Since every source model is smoothed by a Gaussian function, the final ensembled gradient with smoothing is
\begin{equation}
\label{eq:ensemble_smooth_grad}
    \hat{g}(x) = \frac{1}{K} \sum_{k=1}^K \nabla f^k_\sigma(x) 
 = \frac{1}{K\cdot m} \sum_{k=1}^K \sum_{i=1}^m \nabla f^k(x+\sigma u^k_i),
\end{equation}
where $u^k_i \in \mathbb{R}^d$ and $u^k_i \sim \mathcal{N}(0,I_d)$. The following theorem shows that there is an upper bound on the approximation error between $\hat{g}(x)$ and the gradient $\nabla h(x)$ of the target model $h$.

According to Theorem~\ref{thr:ensemble}, the upper bound depends on a constant $C$, which is expected to be small since it bounds the error between the source and target models, which are closely related.

\begin{theorem}
\label{thr:ensemble}
   Suppose the gradient of the target model $h$ is $L$-Lipschitz with $L$ being the Lipschitz constant, and $\hat{g}(x)$ is defined in Eq.~\eqref{eq:ensemble_smooth_grad}, we have
    \begin{equation}
        \Vert \hat{g}(x) - \nabla h(x) \Vert
        \leqslant \left(L\sigma+\frac{C}{\sigma}\right) \sqrt{2} \frac{\Gamma(\frac{d+1}{2})}{\Gamma(\frac{d}{2})},
    \end{equation}
    where $C$ is a constant related to $\{f^1, \dots, f^K\}$ and $\Gamma(\cdot)$ represents the Gamma function~\cite{2002_gamma}.
\end{theorem}

\begin{proof}
    Suppose $\nabla h(x)$ is Lipschitz-continuous with $L$ as the Lipschitz constant.
    According to work~\cite{proof_gradient}, 
    \begin{equation}
        \nabla f_\sigma (x) = \frac{1}{\sigma} \mathbb{E}_{u\sim \mathcal{N}(0,I_d)} [f(x+\sigma u)u].
    \end{equation}
    Moreover, each source model $f^k\in \{f^1, \dots, f^K\}$ can be represented by
    \begin{equation}
        f^k(x) = h(x) + \varepsilon^k(x),
    \end{equation}
    where $\varepsilon^k(x)$ is the error function. Assume the absolute value of each $\varepsilon^k(x)$ is bounded by a constant $c^k$,~\ie $|\varepsilon^k(x)|\leqslant c^k$.
    Let $C = \max\{c^1,\dots,c^K\}$.
    
    Recall that
    \begin{equation}
        \hat{g}(x) = \frac{1}{K} \sum_{k=1}^K \nabla f^k_\sigma (x),
    \end{equation}
    then
    {\small
    \begin{equation}
    \begin{split}
        &\Vert \hat{g}(x) - \nabla h(x) \Vert
        = \left\Vert 
        \frac{1}{K\cdot \sigma} \sum_{k=1}^K \mathbb{E}[f^k(x+\sigma u)u] - \nabla h(x) 
        \right\Vert \\
        &= \left\Vert 
        \frac{1}{K} \sum_{k=1}^K ( \mathbb{E}[h(x+\sigma u)u  / \sigma] + \mathbb{E}[\varepsilon^k(x+\sigma u)u/\sigma]) - \nabla h(x) 
        \right\Vert \\
        &={\left\Vert 
        \frac{1}{K} \sum_{k=1}^K \mathbb{E} [\nabla h(x+\sigma u) -\nabla h(x)]
        + \frac{1}{K} \sum_{k=1}^K \mathbb{E}\left[\frac{\varepsilon^k(x+\sigma u)u}{\sigma}\right]
        \right\Vert }\\
        &\leqslant \frac{1}{K} \sum_{k=1}^K \mathbb{E} \Vert \nabla h(x+\sigma u) -\nabla h(x) \Vert 
        +\frac{1}{K} \sum_{k=1}^K \mathbb{E} \left \Vert \frac{\varepsilon^k(x+\sigma u)u}{\sigma} \right \Vert \\
        &\leqslant \frac{L\cdot \sigma}{K} \sum_{k=1}^K \mathbb{E} \Vert u\Vert + \frac{C}{K\cdot\sigma} \sum_{k=1}^K \mathbb{E} \Vert u\Vert \\
        &= \left(L\sigma+\frac{C}{\sigma}\right)\mathbb{E} \Vert u\Vert
        =\left(L\sigma+\frac{C}{\sigma}\right) \sqrt{2} \frac{\Gamma(\frac{d+1}{2})}{\Gamma(\frac{d}{2})}.
    \end{split}
    \end{equation}
    }
\end{proof}

\begin{figure}[!t]
    \centering
    \includegraphics[width=0.98\linewidth]{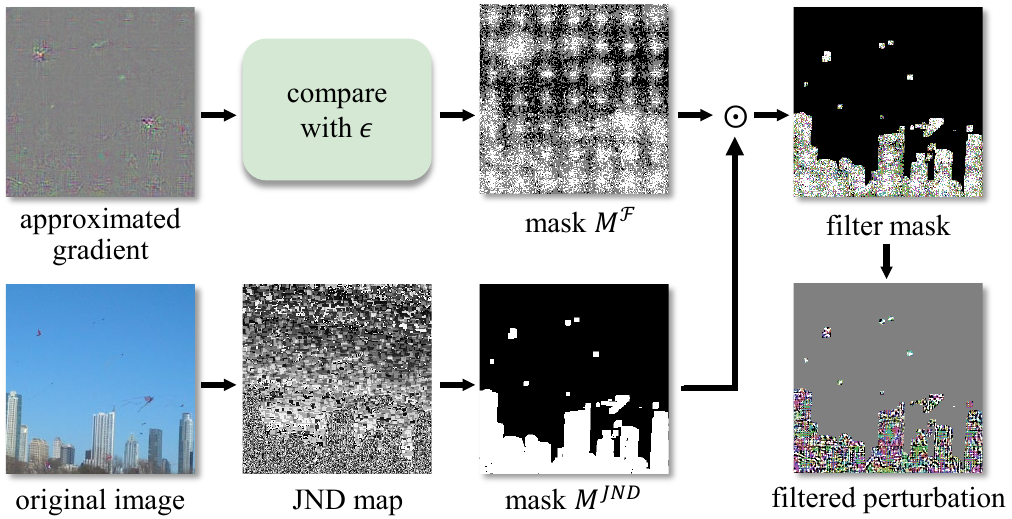}
    \caption{The perturbation filtering pipeline employs two masks. The first mask, $M^{\mathcal{F}}$, removes unimportant components from the approximated gradient. The second mask, $M^{\text{JND}}$, is designed according to the JND map. It filters perturbations on pixels where the JND value is below a threshold $\epsilon$, ensuring perceptual imperceptibility.}
    \label{fig:pert_filter}
\end{figure}

\subsection{Perturbation Filtering}
\label{sec:pert_filter}
Based on Eq.~\eqref{eq:ensemble_smooth_grad} and the basic idea of the FGSM attack~\cite{2015_ICLR_Goodfellow_FGSM}, the adversarial perturbation is

\begin{equation}
\label{eq:FGSM_SEGA}
\Delta x =
\begin{cases}
-\epsilon \cdot \text{sgn}(\hat{g}(x)), & \text{if } h(x)>50, \\
\epsilon \cdot \text{sgn}(\hat{g}(x)), & \text{otherwise}.
\end{cases}
\end{equation}

However, we think this straightforward approach may introduce redundant noise, resulting in poor imperceptibility of the adversarial perturbations. To overcome this limitation, we introduce a perturbation filter module designed to remove these undesirable noise, as illustrated in Fig.~\ref{fig:pert_filter}.

Firstly, we observe that the $\text{sgn}(\cdot)$ function may amplify some \textbf{unimportant components} in the gradient $\hat{g}(x)$. In detail, components with small absolute values in $\hat{g}(x)$ theoretically have limited impact on the predicted score, yet the $\text{sgn}(\cdot)$ function treats them as equally important as those with larger values. For instance, if a component of $\hat{g}(x)$ is $0.0001$, adding perturbations to it has little impact on the predicted score. However, since $\text{sgn}(0.0001) = 1$, Eq.~\eqref{eq:FGSM_SEGA} still applies an $\epsilon$-sized perturbation to this component. Therefore, to remove such kinds of unimportant perturbations, we design a filter mask based on the component values of $\hat{g}$ as follows, where $\alpha$ is a pre-defined threshold. 
\begin{equation}
\label{eq:filter_grad}
M^\mathcal{F}_j =
\begin{cases}
0, & \text{if } |\hat{g}_j| < \alpha, \\
1, & \text{otherwise}.
\end{cases}
\end{equation}
In this equation, $j \in \{1, \dots, d\}$ denotes the index of each component.

\begin{algorithm}[!t]
    \caption{Signed Ensemble Gaussian Attack (SEGA)}
    \label{alg:SEGA}
    \textbf{Inputs}: Source models $f^1,\dots,f^K$, target model $h$, original image $x$, JND mask $\text{JND}(x)$\\
    \textbf{Parameters}: A pre-defined threshold $\alpha$, attack strength $\epsilon$, Gaussian smoothing parameter $\sigma$, sampling number $m$\\
    \textbf{Output}: Adversarial example $\Tilde{x}$
    \begin{algorithmic}[1] %
        \STATE Initialize $\hat{g} \gets 0$
        \FOR{$k=1$ to $K$}
        \FOR{$i=1$ to $m$}
        \STATE $\hat{g} \gets \hat{g} + \nabla f^k(x+\sigma u^k_i),~~\text{where~~}u^k_i\sim \mathcal{N}(0,I_d)$.
        \ENDFOR
        \ENDFOR
        \STATE $\hat{g} \gets \hat{g} / Km$ 
        \STATE $M^{\mathcal{F}} \gets \hat{g} \geqslant \alpha$ \COMMENT{See Eq.~\eqref{eq:filter_grad}}
        \STATE $M^{\text{JND}} \gets \text{JND}(x) \geqslant \epsilon$ \COMMENT{See Eq.~\eqref{eq:JND_grad}}
        \IF {$h(x)> 50$}
        \STATE $\Tilde{x} \gets x - M^{\text{JND}} \odot M^\mathcal{F} \odot \left(\epsilon \cdot \text{sgn}(\hat{g}) \right) $ \COMMENT{decrease score}
        \ELSE
        \STATE $\Tilde{x} \gets x + M^{\text{JND}} \odot M^\mathcal{F} \odot \left(\epsilon \cdot \text{sgn}(\hat{g}) \right)$ \COMMENT{increase score}
        \ENDIF
        \STATE \textbf{return} $\Tilde{x}$
    \end{algorithmic}
\end{algorithm}

Secondly, recognizing that different regions of an image exhibit varying tolerance to perturbations~\cite{2013_QoMEX_Yu_ImageComplexity}, we further remove some \textbf{inappropriate perturbations} based on a Just Noticeable Difference (JND) map. Studies have shown that perturbations in low-frequency regions are more perceptible compared to those in high-frequency regions~\cite{2010_TCSVT_Liu_JNDregions,2024_TCSVT_SurFreeIQA}. The JND value at each pixel serves as a metric to quantify this tolerance. It represents the maximum perturbation that the pixel can tolerate without being perceptible. 
Therefore, we compute the JND map of $x$ using the method proposed by Liu~\etal~\cite{2010_TCSVT_Liu_JNDregions}, and design a mask based on $\text{JND}(x)$ as follows,
\begin{equation}
\label{eq:JND_grad}
\begin{split}
    M^{\text{JND}}_{j} &= 
    \begin{cases} 
        0, & \text{if~~} \text{JND}_{j}(x) < \epsilon, \\
        1, & \text{otherwise}.
    \end{cases} \\
\end{split}
\end{equation}

The final filtered adversarial perturbation takes the form:
\begin{equation}
    \Delta x = \pm \left(M^{\text{JND}} \odot M^\mathcal{F}\right) \odot \left(
    \epsilon \cdot \text{sgn}
        (\hat{g}(x))
        \right),
\end{equation}
where the sign depends on $h(x)$ as Eq.~\eqref{alg:SEGA} shows, $\odot$ represents the element-wise (Hadamard) multiplication.
The complete SEGA algorithm is provided in Algorithm~\ref{alg:SEGA}, where all model predictions are mapped to the interval $[0,100]$.

\section{Experiments}
\label{sec:experiment}
 In this section, we first introduce experimental settings (Sec.~\ref{sec:exp_setting}), including the NR-IQA models, the compared attack methods, the transfer-based black-box settings, and the hyperparameters of attacks. Next, we report the superior transferability of SEGA compared to existing transfer-based attacks in Sec.~\ref{sec:main_result}. We also demonstrate the imperceptibility of the perturbations using both quantitative and qualitative metrics (Sec.~\ref{sec:exp_quality}). 
 Ablation studies in Sec.~\ref{sec:exp_ablation} validate the choice of hyperparameters in SEGA, while Sec.~\ref{sec:complexity} analyzes SEGA's computational complexity and efficiency. Finally, in Sec.~\ref{appendix:exp_query}, we provide a comparative analysis between SEGA and query-based attacks in the NR-IQA context.

 \begin{table*}[!t]
\caption{Comparison of SEGA with other attacks on the CLIVE dataset. Under the best strategy, we present the highest performance achieved by the compared method across all source models, with the associated source model listed in the `source' column. The highest and second-highest performance values are highlighted in \textbf{bold} and \underline{underline} respectively}
    \centering
    \renewcommand\arraystretch{1.1}
    \setlength{\tabcolsep}{1.7pt}
    \resizebox{\textwidth}{!}{
    \begin{tabular}{clcccccccccccc}
    \toprule
         & & \multicolumn{6}{c}{Target: HyperIQA} & \multicolumn{6}{c}{Target: DBCNN}\\ \cmidrule(lr){3-8} \cmidrule(lr){9-14}
        &~ & Source & MAE$\uparrow$ & ~R$\downarrow~$ & SROCC$\downarrow$ & PLCC$\downarrow$ & KROCC$\downarrow$ & Source & MAE$\uparrow$ & R$\downarrow$ & SROCC$\downarrow$ & PLCC$\downarrow$ & KROCC$\downarrow$\\ \midrule
        \multirow{4}{*}{Best Strategy} & FGSM & DBCNN & 8.648 & 1.054 & 0.825 & 0.841 & 0.640 
        & LIQE & 7.950 & 1.003 & 0.866 & 0.859 & 0.680 \\ 
        
        &Pattack & Linearity & 7.992 & 1.153 & 0.941 & 0.923 & 0.793 
        & Linearity & 6.423 & 1.148 & 0.968 & 0.966 & 0.852 \\ 
        
        &OUAP & Linearity & \textbf{16.779} & \underline{0.850} & \underline{0.733} & \textbf{0.692} & \underline{0.552} 
        & LIQE & \textbf{13.282} & \underline{0.894} & \underline{0.747} & \underline{0.731} & \underline{0.557} \\ 
        
        &IOI & Linearity & 7.728 & 1.120 & 0.931 & 0.921 & 0.781 
        & Linearity & 5.974 & 1.262 & 0.931 & 0.912 & 0.770 \\ \hline

        \multirow{4}{*}{Average Strategy} & FGSM & - & 8.929 & 1.024 & 0.840 & 0.855 & 0.664 & - & 7.166 & 1.071 & 0.866 & 0.861 & 0.682  \\ 
        &Pattack & - & 7.375 & 1.167 & 0.959 & 0.948 & 0.834 & - & 5.058 & 1.279 & 0.974 & 0.975 & 0.870 \\ 
        &OUAP & - & 4.205 & 1.498 & 0.965 & 0.954 & 0.848 & - & 2.912 & 1.548 & 0.978 & 0.975 & 0.869  \\  
        &IOI & - & 5.524 & 1.292 & 0.973 & 0.966 & 0.864 & - & 4.290 & 1.401 & 0.970 & 0.962 & 0.858  \\ \hline
        
        &Kor & - & \underline{16.505} & \textbf{0.781} & 0.858 & 0.842 & 0.675 
        & - & 8.914 & 1.028 & 0.883 & 0.869 & 0.710 \\ 
        &SEGA & - & 12.186 & 0.883 & \textbf{0.675} & \underline{0.706} & \textbf{0.507} 
        & - & \underline{10.493} & \textbf{0.876} & \textbf{0.562} & \textbf{0.626} & \textbf{0.412} \\ \midrule
        
        &~ & \multicolumn{6}{c}{Target: LinearityIQA} & \multicolumn{6}{c}{Target: LIQE}\\ \cmidrule(lr){3-8} \cmidrule(lr){9-14}
        & ~ & Source & MAE$\uparrow$ & R$\downarrow$ & SROCC$\downarrow$ & PLCC$\downarrow$ & KROCC$\downarrow$ & Source & MAE$\uparrow$ & R$\downarrow$ & SROCC$\downarrow$ & PLCC$\downarrow$ & KROCC$\downarrow$\\ \midrule
        \multirow{4}{*}{Best Strategy} & FGSM & HyperIQA & 12.383 & 0.969 & 0.709 & 0.666 & 0.531 
        & HyperIQA & \underline{11.344} & \underline{1.103} & 0.909 & 0.894 & 0.754 \\ 
        
        &Pattack & HyperIQA & 7.896 & 1.156 & 0.953 & 0.930 & 0.824 
        & HyperIQA & 6.145 & 1.440 & 0.976 & 0.979 & 0.876 \\ 
        
        &OUAP & DBCNN & \underline{15.789} & 0.966 & \underline{0.663} & \underline{0.596} & \underline{0.501}
        & DBCNN & 6.406 & 1.450 & 0.938 & 0.941 & 0.826 \\
        
        &IOI & HyperIQA & 9.133 & 1.123 & 0.870 & 0.865 & 0.693
        & HyperIQA & 7.098 & 1.437 & 0.956 & 0.948 & 0.837 \\ \hline

        \multirow{4}{*}{Average Strategy} & FGSM & - & 12.787 & \underline{0.880} & 0.790 & 0.754 & 0.602 & - & 8.133 & 1.267 & 0.952 & 0.943 & 0.825  \\ 
        &Pattack & - & 8.042 & 1.126 & 0.962 & 0.941 & 0.834 & - & 5.781 & 1.477 & 0.980 & 0.978 & 0.889 \\ 
        &OUAP & - & 5.583 & 1.366 & 0.947 & 0.931 & 0.808 & - & 3.004 & 1.795 & 0.990 & 0.991 & 0.923  \\ 
        &IOI & - & 7.114 & 1.210 & 0.946 & 0.935 & 0.811 & - & 5.102 & 1.689 & 0.979 & 0.973 & 0.896  \\ \hline
        
        &Kor & - & 14.475 & 0.913 & 0.802 & 0.786 & 0.612 
        & - & 10.742 & 1.311 & \underline{0.899} & \underline{0.883} & \underline{0.752} \\ 
        
        &SEGA & - & \textbf{17.938} &\textbf{0.700} & \textbf{0.504} & \textbf{0.503} & \textbf{0.366} 
        & - & \textbf{15.094} & \textbf{0.927} & \textbf{0.831} & \textbf{0.801} & \textbf{0.663} \\ \bottomrule
    \end{tabular}
    }
\label{tab:main_result_clive}
\end{table*}

 \begin{table*}[!t]
\caption{
Comparison of SEGA with other attacks on the KonIQ-10k dataset. Under the best strategy, we present the highest performance achieved by the compared method across all source models, with the associated source model listed in the `source' column. The highest and second-highest performance values are highlighted in \textbf{bold} and \underline{underline} respectively
}
    \centering
    \renewcommand\arraystretch{1.1}
    \setlength{\tabcolsep}{1.7pt}
    \resizebox{\textwidth}{!}{
    \begin{tabular}{clcccccccccccc}
    \toprule
         & & \multicolumn{6}{c}{Target: HyperIQA} & \multicolumn{6}{c}{Target: DBCNN}\\ \cmidrule(lr){3-8} \cmidrule(lr){9-14}
        &~ & Source & MAE$\uparrow$ & ~R$\downarrow~$ & SROCC$\downarrow$ & PLCC$\downarrow$ & KROCC$\downarrow$ & Source & MAE$\uparrow$ & R$\downarrow$ & SROCC$\downarrow$ & PLCC$\downarrow$ & KROCC$\downarrow$\\ \midrule
        \multirow{4}{*}{Best Strategy} 
        & FGSM & LIQE & \underline{12.902} & \textbf{0.798} & 0.764 & 0.782 & 0.569 & LIQE & \underline{13.165} & \textbf{0.765} & \underline{0.737} & 0.747 & \underline{0.539} \\ 
        &Pattack & DBCNN & 3.544 & 1.505 & 0.933 & 0.936 & 0.781 & LIQE & 4.383 & 1.280 & 0.963 & 0.971 & 0.836 \\ 
        &OUAP & DBCNN & 8.751 & 1.254 & \textbf{0.626} & \textbf{0.560} & \textbf{0.462} & HyperIQA & 7.281 & 1.242 & 0.758 & \underline{0.701} & 0.568\\ 
        &IOI & LinearityIQA & 5.183 & 1.333 & 0.856 & 0.875 & 0.673 & LinearityIQA & 6.107 & 1.182 & 0.869 & 0.886 & 0.687 \\ \hline
        
        \multirow{4}{*}{Average Strategy} 
        & FGSM & - & 3.859 & 1.364 & 0.966 & 0.963 & 0.845 & - & 4.634 & 1.260 & 0.958 & 0.951 & 0.821  \\ 
        &Pattack & - & 3.254 & 1.578 & 0.961 & 0.956 & 0.834 & - & 2.989 & 1.535 & 0.977 & 0.981 & 0.873 \\ 
        &OUAP & - & 4.598 & 1.336 & 0.933 & 0.948 & 0.784 & - & 4.947 & 1.254 & 0.937 & 0.946 & 0.785 \\  
        &IOI & - & 2.790 & 1.643 & 0.958 & 0.956 & 0.828 & - & 3.138 & 1.536 & 0.950 & 0.954 & 0.810  \\ \hline
        
        &Kor & - & 9.165 & 1.164 & 0.774 & 0.756 & 0.591 & - & 5.570 & 1.342 & 0.869 & 0.853 & 0.691  \\ 
        &SEGA & - & \textbf{13.292} & \underline{0.818} & \underline{0.653} & \underline{0.664} & \underline{0.470}
            & - & \textbf{14.111} & \underline{0.782} & \textbf{0.584} & \textbf{0.606} & \textbf{0.414} \\ \midrule
        
        &~ & \multicolumn{6}{c}{Target: LinearityIQA} & \multicolumn{6}{c}{Target: LIQE}\\ \cmidrule(lr){3-8} \cmidrule(lr){9-14}
        & ~ & Source & MAE$\uparrow$ & R$\downarrow$ & SROCC$\downarrow$ & PLCC$\downarrow$ & KROCC$\downarrow$ & Source & MAE$\uparrow$ & R$\downarrow$ & SROCC$\downarrow$ & PLCC$\downarrow$ & KROCC$\downarrow$\\ \midrule
        \multirow{4}{*}{Best Strategy} 
        & FGSM & HyperIQA & 4.798 & \underline{1.400} & 0.710 & 0.675 & 0.518
        & LinearityIQA & \textbf{24.891} & \textbf{0.730} & \underline{0.826} & \underline{0.811} & \underline{0.636}  \\
        
        &Pattack & HyperIQA & 2.541 & 1.815 & 0.925 & 0.898 & 0.770
        & LinearityIQA & 6.184 & 1.455 & 0.940 & 0.946 & 0.807 \\ 
        
        &OUAP & DBCNN & \textbf{6.351} & 1.418 & \textbf{0.560} & \textbf{0.516} & \textbf{0.398}
        & LinearityIQA & 16.734 & 0.892 & 0.899 & 0.884 & 0.730 \\
        
        &IOI & HyperIQA & 3.338 & 1.651 & 0.837 & 0.818 & 0.652
        & LinearityIQA & 6.059 & 1.470 & 0.965 & 0.962 & 0.846  \\ \hline

        \multirow{4}{*}{Average Strategy} 
        & FGSM & - & 2.080 & 1.817 & 0.949 & 0.931 & 0.813 & - & 3.311 & 1.707 & 0.989 & 0.986 & 0.918 \\
        &Pattack & - & 2.650 & 1.768 & 0.940 & 0.914 & 0.794 & - & 5.262 & 1.541 & 0.953 & 0.958 & 0.831 \\
        &OUAP & - & 2.542 & 1.811 & 0.902 & 0.871 & 0.738 & - & 7.664 & 1.359 & 0.954 & 0.954 & 0.832 \\
        &IOI & - & 2.358 & 1.855 & 0.929 & 0.910 & 0.784 & - & 2.834 & 1.786 & 0.989 & 0.988 & 0.919 \\ \hline
        
        &Kor & - & \underline{5.303} & 1.485 & \underline{0.704} & \underline{0.660} & \underline{0.524}
        & - & 7.231 & 1.417 & 0.932 & 0.927 & 0.784 \\
        
        & SEGA & - & 5.164 & \textbf{1.332} & 0.800 & 0.756 & 0.606
        & - & \underline{22.094} & \underline{0.746} & \textbf{0.809} & \textbf{0.767} & \textbf{0.622}  \\ \bottomrule
    \end{tabular}
    }
\label{tab:main_result_koniq}
\end{table*}

\begin{table*}[!t]
\caption{Quantitative comparison of the image quality of adversarial examples generated by different attack methods against various target models, with SSIM values computed between the adversarial examples and the original images}
    \centering
    \resizebox{0.9 \textwidth}{!}{
    \begin{tabular}{lcccccccc}
    \toprule
         & \multicolumn{4}{c}{CLIVE} & \multicolumn{4}{c}{KonIQ-10k} \\ \cmidrule(lr){2-5} \cmidrule(lr){6-9}
        Target & HyperIQA & DBCNN & LinearityIQA & LIQE & HyperIQA & DBCNN & LinearityIQA & LIQE \\ \midrule
        FGSM & 0.733 & 0.755 & 0.740 & 0.740 & 0.748 & 0.748 & 0.741 & 0.739 \\
        Pattack & 0.731 & 0.731 & 0.733 & 0.733 & 0.740 & 0.739 & 0.740 & 0.740\\ 
        OUAP & 0.815 & 0.803 & 0.826 & 0.826 & 0.821 & 0.836 & 0.821 & 0.809\\ 
        IOI & 0.883 & 0.883 & 0.888 & 0.888 & 0.905 & 0.905 & 0.906 & 0.905\\ 
        Kor & 0.931 & 0.931 & 0.931 & 0.931 & 0.940 & 0.940 & 0.940 & 0.940\\ 
        SEGA & 0.862 & 0.857 & 0.881 & 0.863 & 0.837 & 0.836 & 0.840 & 0.838\\ \bottomrule
    \end{tabular}
    }
\label{tab:image_SSIM}
\end{table*}

\subsection{Experimental Settings}
\label{sec:exp_setting}
\textbf{Dataset and NR-IQA models.} 
Experiments were conducted on the widely used CLIVE~\cite{2015_TIP_LIVEC} and KonIQ-10k~\cite{koniq_10k} datasets. For CLIVE, we randomly split the data in 80\%/20\% for training and adversarial attack evaluation, following the protocol of prior work~\cite{2024_CVPR_Liu_NT}. For the larger KonIQ-10k dataset, we employed a 60\%/20\%/20\% split for training, validation, and testing.
We consider five representative NR-IQA models: HyperIQA~\cite{2020_CVPR_hyperIQA}, DBCNN~\cite{2020_TCSVT_DBCNN}, LinearityIQA~\cite{2020_MM_LinearityIQA}, LIQE~\cite{2023_CVPR_LIQE}, and MANIQA~\cite{2022_CVPRw_MANIQA}. Among them, the first three are CNN-based, while the latter two are transformer-based. All models were trained on the same training subset using the official implementations released by their respective authors.

\textbf{Compared attack methods.} We compare our method with five NR-IQA attack methods: FGSM adapted for NR-IQA models~\cite{2024_CVPR_Liu_NT}, perceptual attack (Pattack)~\cite{2022_NIPS_Zhang_PAttack}, OUAP~\cite{2024_CVIU_Ekaterina_OUAP}, IOI~\cite{2024_ICML_Ekaterina_IOI}, and the attack proposed by Korhonen~\etal~\cite{2022_QEVMAw_Korhonen_BIQA} (short for Kor). All these methods are applicable in the transfer-based scenario. The one-step FGSM is implemented following the setup in work~\cite{2024_CVPR_Liu_NT}, while the other attack methods use the official implementations released by their respective authors.

\textbf{Transfer-based black-box settings.} The main experiment contains HyperIQA~\cite{2020_CVPR_hyperIQA}, DBCNN~\cite{2020_TCSVT_DBCNN}, LinearityIQA~\cite{2020_MM_LinearityIQA}, and LIQE~\cite{2023_CVPR_LIQE} models. For each experimental run, one model serves as the target model $h$, while the remaining three models form an ensemble of source models $\{f^1, f^2, f^3\}$ for our proposed SEGA method. Following the setup in~\cite{2022_QEVMAw_Korhonen_BIQA}, the Kor method employs a pre-trained ResNet as the source model, regardless of the target model identity.

For other competing methods, we adopt two evaluation strategies: \emph{Best Strategy} and \emph{Average Strategy}. Under the \emph{Best strategy}, each of the other three models is independently used as the source model. We evaluate the attack performance for each $\{f^k$, $h\}$ pair ($k = 1, 2, 3$) and report the best performance. The \emph{Average Strategy}, on the other hand, constructs the final adversarial example for each method by averaging the perturbations generated by that method across all available source models.

\textbf{Parameters chosen.} To enable a fair comparison, we match the attack strength (the $\ell_\infty$ norm of the adversarial perturbation) for all methods that expose such a parameter, and adopt the established protocols from previous studies for the remaining methods. We configure each attack method as follows. We use the one-step FGSM attack with an attack strength of $0.03$~\cite{2024_CVPR_Liu_NT}. Following the experimental settings in~\cite{2024_TCSVT_SurFreeIQA}, Pattack~\cite{2018_CVPR_Richard_LPIPS} adopts LPIPS as the perceptual constraint, with the corresponding weight set to $9,000,000$. OUAP generates the universal perturbation using $10$ iterations with an attack strength of $0.03$ trained on the testing data~\cite{2024_CVIU_Ekaterina_OUAP}. IOI uses a learning rate of $0.1$ and a truncation parameter of $0.07$~\cite{2024_ICML_Ekaterina_IOI}. For the Kor attack, the learning rate is set to $2$. The proposed SEGA method sets the attack strength to $\epsilon=0.03$. For Gaussian smoothing, the sample number is $m=10$, and the smoothing parameter is $\sigma=10/255$. 
The filter mask $M^\mathcal{F}$ employs thresholds of $\alpha=0.02$ for the CLIVE dataset and $\alpha=0.005$ for the KonIQ-10k dataset.

Based on our experimental experience detailed in the following sections, we provide guidance for setting SEGA's hyperparameters to balance transferability, computational cost, and perturbation imperceptibility. For standard use cases aiming at a practical trade-off, we recommend $\sigma=10/255$ and $m=10$. If computational time is less constrained and higher transferability is desired, increasing $m$ to $20$ is advised. Regarding the perturbation imperceptibility, the filter threshold $\alpha=0.02$ typically yields adversarial examples with an SSIM around 0.9, while $\alpha=0.005$ generally produces examples with an SSIM near 0.85.

\subsection{Comparison in Transferability}
\label{sec:main_result}
In this section, we evaluate the transferability of SEGA in comparison with SOTA attack methods on the CLIVE dataset and KonIQ-10k datasets. Transferability is assessed using five metrics calculated between predicted scores before and after attacks: Mean Absolute Error (MAE), R robustness~\cite{2022_NIPS_Zhang_PAttack}, Spearman Rank Order Correlation Coefficient (SROCC), Pearson's Linear Correlation Coefficient (PLCC), and Kendall's Rank-Order Correlation Coefficient (KROCC). Higher MAE and lower values for the other metrics indicate better transferability. The main results are shown in Table~\ref{tab:main_result_clive} and Table~\ref{tab:main_result_koniq}.

\textbf{Under the best strategy evaluation,} as shown in Table~\ref{tab:main_result_clive}, SEGA outperforms other methods across nearly all metrics and target models on the CLIVE dataset, particularly in SROCC, PLCC, and KROCC. In every case, SEGA achieves the lowest SROCC after attacks, indicating its superior ability to disrupt prediction consistency compared to other methods. For instance, when attacking DBCNN, SEGA achieves an SROCC of 0.56, whereas the second-best method (OUAP) achieves only around 0.75. In contrast, the other methods exhibit weak attack performance across most metrics. When attacking DBCNN, the MAE values for the compared methods except OUAP are all below 10, and their SROCC values remain above 0.85, highlighting their poor transferability from source models to the target model. When the target model is LinearityIQA or LIQE, SEGA achieves the best attack performance in terms of both MAE and R robustness, demonstrating its strong capability to manipulate prediction accuracy.

As shown in Table~\ref{tab:main_result_koniq}, SEGA consistently achieves either the best or second-best performance across most evaluation settings. For example, when attacking the DBCNN model, SEGA reduces the SROCC to 0.584 and the PLCC to 0.606, significantly lower than the best-performing baseline (FGSM: SROCC 0.737, PLCC 0.747). In addition, SEGA attains an MAE of 14.111, notably surpassing all compared methods (e.g., OUAP: 7.281), indicating a substantially greater deviation from the original predictions and thus stronger transferability.

While SEGA generally achieves strong performance across both tables, there are a few exceptions where it does not outperform existing methods, particularly when compared to OUAP in certain scenarios.
We attribute this to the relatively noticeable perturbations produced by OUAP, as discussed in the next section. Additionally, while SEGA’s perturbation filtering enhances imperceptibility by removing certain noise, it may slightly reduce attack effectiveness.

\textbf{Under the average strategy evaluation,} ensembling adversarial perturbations does not improve the transferability of the compared methods on both datasets. In most cases, SROCC remains above 0.9, reflecting limited transferability and highlighting the effectiveness of SEGA’s ensemble of Gaussian-smoothed gradients.
SEGA consistently surpasses all baseline methods across nearly all scenarios. For example, when attacking DBCNN on the CLIVE dataset, SEGA is the only method achieving an R robustness below 1. In contrast, all other methods yield SROCC values above 0.85, revealing their limited capability to disrupt prediction consistency. SEGA, however, achieves an SROCC of 0.5622, demonstrating its strong ability to degrade the target DBCNN model’s predictions. A similar trend is observed on the KonIQ-10k dataset. For instance, when attacking HyperIQA, the average strategy method yields an MAE value less than 5 for all compared methods, whereas SEGA achieves an MAE exceeding 11, highlighting its superior effectiveness.

The experimental results highlight the superiority of SEGA when the compared methods employ the average strategy. This demonstrates the effectiveness of combining Gaussian smoothing, gradient ensembling, and perturbation filtering in SEGA to achieve enhanced transferability.

From Table~\ref{tab:main_result_clive} and Table~\ref{tab:main_result_koniq}, \textbf{we further observe the effectiveness of SEGA in cross-architecture scenarios}. In particular, when attacking the transformer-based LIQE model~\cite{dosovitskiy2020vit} using CNN-based source models, all compared methods suffer a significant drop in transferability, especially in rank-based metrics such as SROCC, PLCC, and KROCC.

This phenomenon is particularly evident for OUAP, whose performance varies across different target architectures. Since OUAP leverages information from all test images to construct a universal perturbation, it exhibits relatively stable transferability among target models with CNN-based architectures. However, the resulting architecture-agnostic perturbation, optimized on CNN source models, may not align well with the input sensitivity and representation characteristics of transformer-based models, leading to unstable and degraded transfer performance. We further support this explanation by measuring the gradient alignment between CNN source models and LIQE in the supplementary material.

Nevertheless, SEGA remains the most effective method, achieving a KROCC of 0.663 on the CLIVE dataset and 0.622 on the KonIQ-10k dataset, while the second-best method achieves only 0.752 and 0.636, respectively. This demonstrates SEGA’s strong ability to transfer adversarial examples from CNN-based models to transformer-based architectures.

\begin{figure*}[!t]
    \centering
    \includegraphics[width=0.98\textwidth]{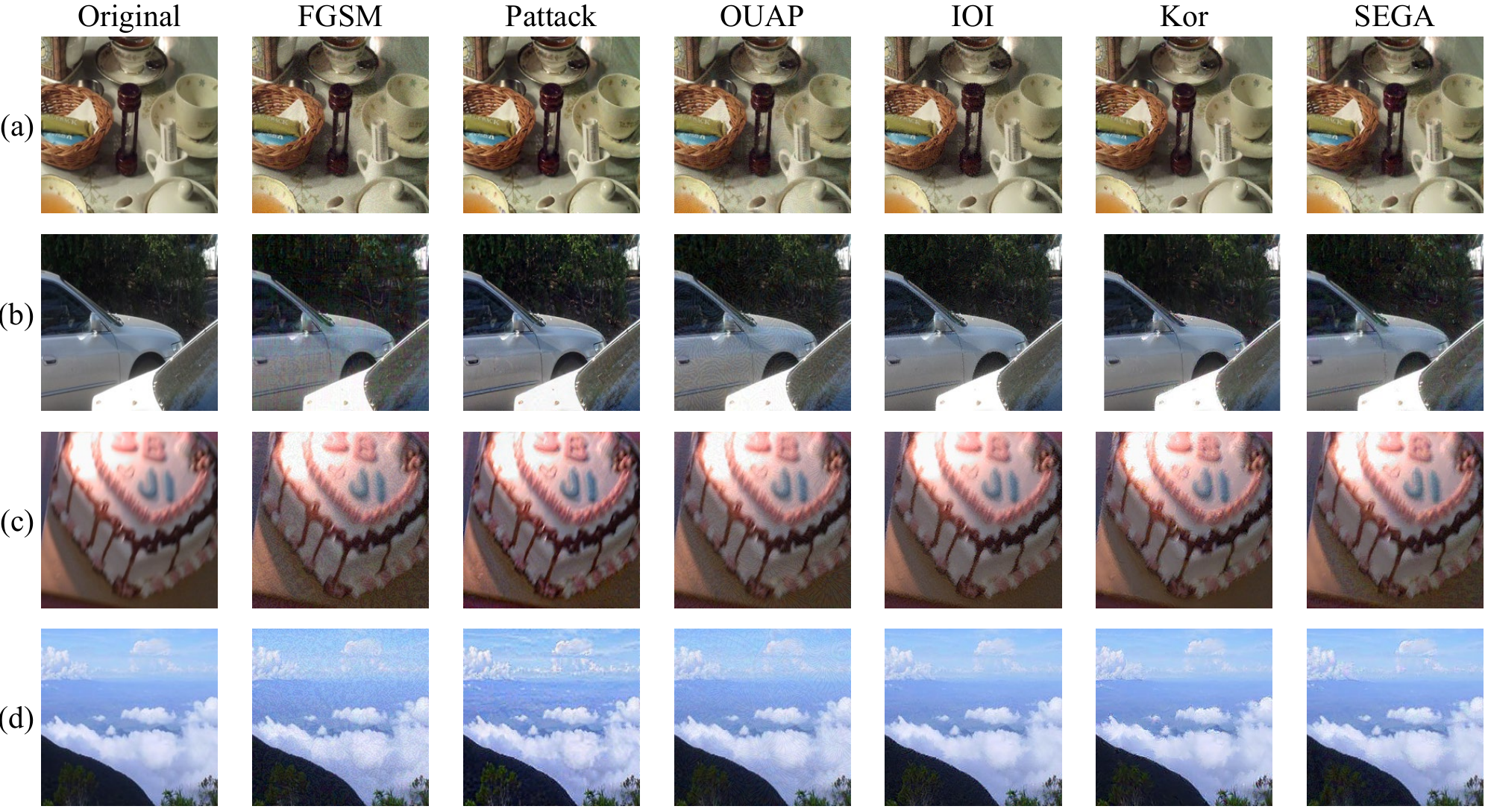}
    \caption{Visualization of adversarial examples generated by different attack methods against different target models on the CLIVE dataset, (a) HyperIQA, (b) DBCNN, (c) LinearityIQA, (d) LIQE. Please zoom in for a clearer view.}
    \label{fig:visualization}
\end{figure*}

\subsection{Imperceptibility of Adversarial Perturbations}
\label{sec:exp_quality}

Since the imperceptibility of adversarial perturbations is a crucial requirement for attacks, we provide both quantitative and qualitative comparisons in this section. As the compared methods perform better under their respective best strategy than under the average strategy (as shown in Table~\ref{tab:main_result_clive}), we evaluate the image quality of their adversarial examples using the best strategy.

The average SSIM~\cite{2004_TIP_Wang_SSIM} values between adversarial examples and the original input images are reported in Table~\ref{tab:image_SSIM}. Higher SSIM values indicate greater similarity between adversarial examples and the original images, suggesting that the adversarial perturbations are less perceptible.
From Table~\ref{tab:image_SSIM}, we observe that SEGA demonstrates significantly better imperceptibility compared to FGSM, Pattack, and OUAP on both CLIVE and KonIQ-10k datasets. Furthermore, SEGA achieves imperceptibility comparable to that of IOI and the Kor attack across all target models. For instance, when attacking HyperIQA on the CLIVE dataset, the SSIM value of SEGA is only 0.02 lower than that of IOI but exceeds those of FGSM and Pattack by more than 0.12.

To provide a more intuitive comparison of the imperceptibility of these attacks, we present some visualization results in Fig.~\ref{fig:visualization}. From Fig.~\ref{fig:visualization}, it is evident that adversarial perturbations in the examples generated by SEGA are less perceptible compared to those generated by FGSM and OUAP, particularly in images with extensive low-frequency regions, such as Fig.~\ref{fig:visualization} (d).
Although SEGA falls slightly behind IOI and the Kor attack in SSIM values, the adversarial examples generated by these methods appear perceptually similar to the human visual system.

\begin{table}[!t]
\caption{Ablation studies on the impact of Gaussian smoothing and gradient ensembling on transferability}
    \centering
    \setlength{\tabcolsep}{3.5pt}
    \resizebox{0.98\linewidth}{!}{
    \begin{tabular}{llccccc}
    \toprule
         Gauss. & Ensemb. & MAE$\uparrow$ & R$\downarrow$ & SROCC$\downarrow$ & PLCC$\downarrow$ & KROCC$\downarrow$ \\ \midrule
        $\times$ & $\times$ & 7.0034 & 1.1186 & 0.8357 & 0.8487 & 0.6553 \\ 
        $\times$ & \checkmark & 7.8475 & 1.0485 & 0.7858 & 0.8024 & 0.6036 \\ 
        \checkmark & $\times$ & 9.1340 & 0.9757 & 0.6698 & 0.7296 & 0.4998 \\ 
        \checkmark & \checkmark & \textbf{10.3865} & \textbf{0.8823} & \textbf{0.5714} & \textbf{0.6421} & \textbf{0.4233} \\ 
    \bottomrule
    \end{tabular}
    }
\label{tab:ab_overview}
\end{table}

\subsection{Ablation Studies}
\label{sec:exp_ablation}

In this section, we conduct ablation studies on the CLIVE dataset to evaluate the three components of SEGA: Gaussian smoothing, gradient ensembling, and perturbation filtering. First, we explore the effectiveness of the first two components on transferability as well as the parameters chosen for these parts. Then, we evaluate the impact of the perturbation filtering on imperceptibility and transferability. All experiments in this section attack the DBCNN as the target model.

\textbf{Effectiveness of Gaussian smoothing and gradient ensembling.} Table~\ref{tab:ab_overview} presents the ablation study on Gaussian smoothing and gradient ensembling. In the table, ``Gauss. ($\times$)" indicates that for a source model $f$, the gradient $\nabla f(x)$ is directly used to guide the generation of adversarial examples, instead of using the gradient of the Gaussian-smoothed function. And ``Ens. ($\times$)" denotes that only HyperIQA is used as the source model, instead of the full set of source models \{HyperIQA, LinearityIQA, LIQE\}. Table~\ref{tab:ab_overview} demonstrates that both components contribute to improving the transferability of adversarial examples, with their combination yielding the best results. Furthermore, Gaussian smoothing proves to be more effective than gradient ensembling in enhancing transferability. 

\begin{table*}[!t]
    \caption{Ablation studies on how perturbation filtering masks $M^\mathcal{F}$ and $M^\text{JND}$ affect both the imperceptibility of adversarial perturbations and the transferability of adversarial examples}
    \centering
    \setlength{\tabcolsep}{3.5pt}
    \resizebox{0.8\textwidth}{!}{
    \begin{tabular}{llccccccccc}
        \toprule
         & & \multicolumn{4}{c}{Imperceptibility} & \multicolumn{5}{c}{Transferability} \\ \cmidrule(lr){3-6} \cmidrule(lr){7-11}
            $M^\mathcal{F}$ & $M^\text{JND}$ &  $\ell_1 \downarrow$ & SSIM $\uparrow$ & LPIPS $\downarrow$& DISTS $\downarrow$
            & MAE $\uparrow$ & R $\downarrow$ & SROCC $\downarrow$ & PLCC $\downarrow$ & KROCC $\downarrow$ \\ \midrule
            $\times$ & $\times$ & 7.838 & 0.749 & 0.314 & 0.202 &  11.175 & 0.844 & 0.522 & 0.582 & 0.381 \\ 
            $\times$ & \checkmark & 5.511 & 0.808 & 0.290 & 0.188 & 12.048 & 0.839 & 0.629 & 0.632 & 0.450\\ 
            \checkmark & $\times$ & 6.501 & 0.817 & 0.261 & 0.181 & 12.749 & 0.847 & 0.684 & 0.684 & 0.497\\ 
            \checkmark & \checkmark & 4.538 & 0.857 & 0.235 & 0.167 & 10.493 & 0.876 & 0.562 & 0.626 & 0.418 \\ 
        \bottomrule
        \end{tabular}
    }
\label{tab:filter_quality_and_trans}
\end{table*}

\begin{figure}[!t]
    \centering
    \includegraphics[width=0.98\linewidth]{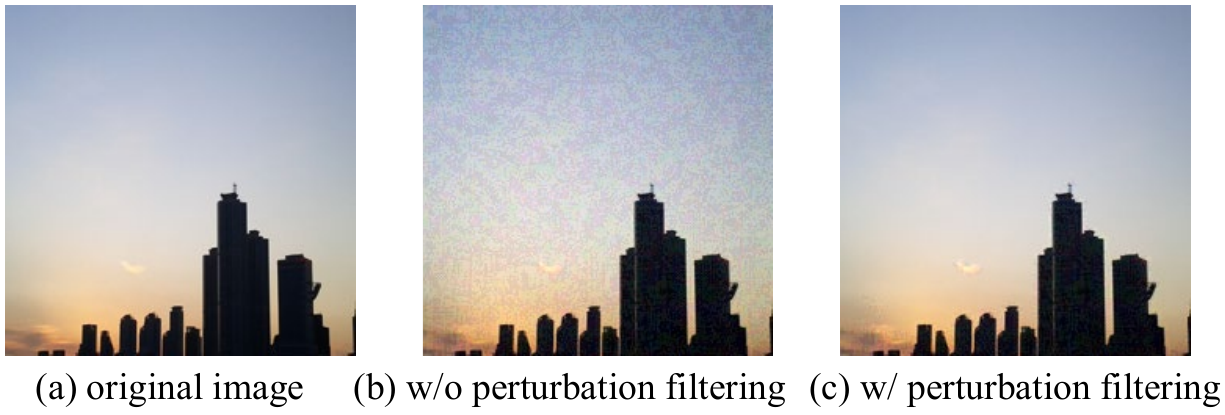}
    \caption{A visualization result of adversarial examples generated with and without the use of perturbation filtering.}
    \label{fig:visual_JND}
\end{figure}

\textbf{Effectiveness of perturbation filtering.}
Table~\ref{tab:filter_quality_and_trans} reports ablation results analyzing how the filter mask $M^\mathcal{F}$ and the JND mask $M^{\text{JND}}$ affect the imperceptibility and the transferability of adversarial perturbations under the $\ell_\infty$ norm constraint. Imperceptibility is evaluated by the $\ell_1$ norms of the perturbations, along with widely used perceptual metrics—SSIM, LPIPS, and DISTS~\cite{2020_TPAMI_Ding_DISTS}—calculated between the adversarial examples and the corresponding original images. The results shown in Table~\ref{tab:filter_quality_and_trans} demonstrate that both masks significantly enhance the imperceptibility of the perturbations. Moreover, Fig.~\ref{fig:visual_JND} provides a visual example illustrating their combined effect on perturbation imperceptibility.
The effect of perturbation filtering on transferability is presented in the right section of Table~\ref{tab:filter_quality_and_trans}. As the filtering process removes specific perturbation components, a slight reduction in transferability is observed when both masks are applied, compared to the unfiltered case. However, we consider this reduction acceptable given the substantial improvement in imperceptibility. For example, the robustness score R differs by approximately 0.03 with and without filtering.
These findings highlight the necessity of incorporating perturbation filtering in SEGA, which slightly reduces transferability but significantly enhances the image quality of adversarial examples.

\textbf{The smoothing parameter $\sigma$ in Gaussian smoothing.} Fig.~\ref{fig:ab_sigma_m_K} (a) illustrates the impact of the smoothing parameter $\sigma$ on transferability. The results indicate that transferability initially improves but then declines as $\sigma$ increases. For instance, the PLCC value decreases from approximately 0.7 to 0.6 as $\sigma$ increases from $5/255$ to $10/255$, and then rises to nearly 0.8 as $\sigma$ increases to $25/255$. Similar trends also occur on other metrics. This phenomenon suggests that $\sigma$ represents a trade-off between smoothness and transferability. On one hand, smaller $\sigma$ limits the effectiveness of noise removal and thus reduces transferability. On the other hand, a larger $\sigma$ results in broader smoothing of $f$, which may introduce large approximation errors. As a result, $f_\sigma$ may no longer serve as a good approximation of the target model. In our main experiments, we select $\sigma = 10 / 255$.

\begin{figure}[!t]
    \centering
    \includegraphics[width=0.98\linewidth]{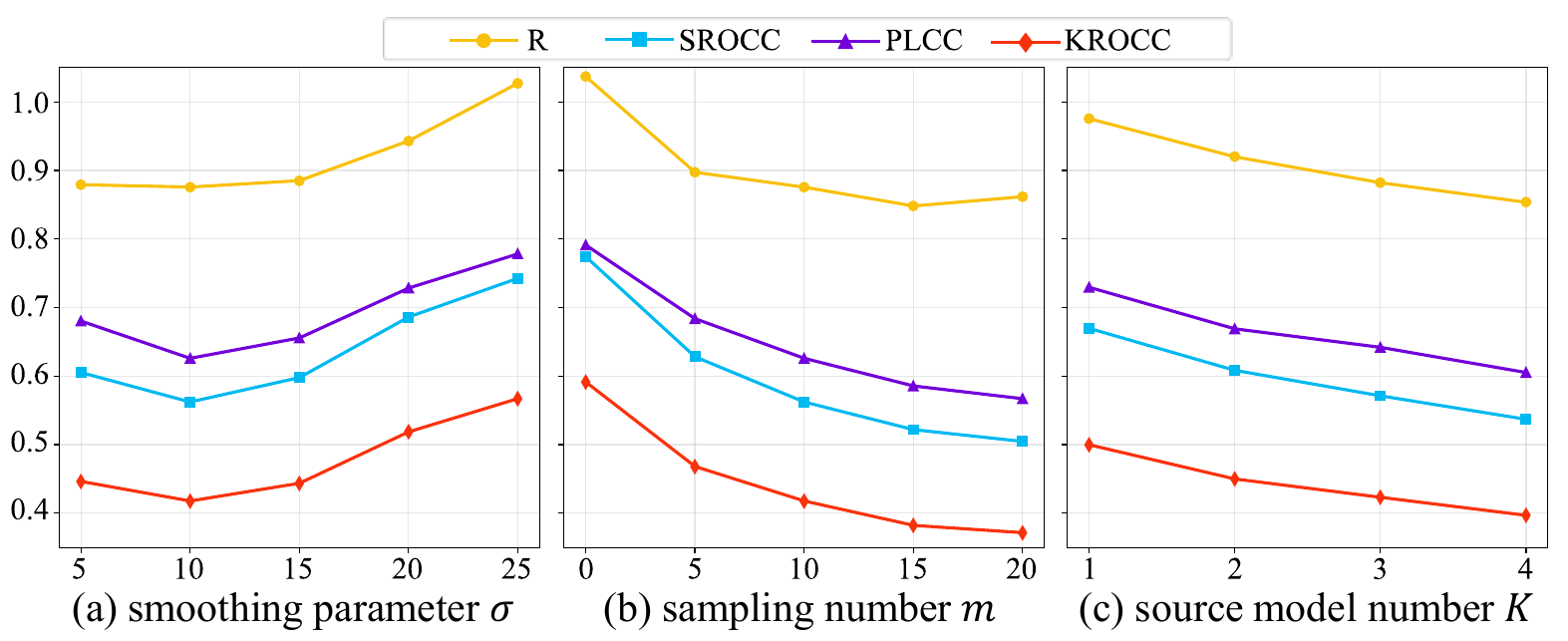}
    \caption{Analysis of hyperparameters through ablation studies and the impact of perturbation filtering on transferability.}
    \label{fig:ab_sigma_m_K}
\end{figure}

\textbf{The sampling number $m$ in Gaussian smoothing.} Since explicitly calculating the gradient of $f_\sigma$ is challenging, we instead use the expectation as shown in Eq.~\eqref{eq:smooth_gradient}. The number of sampling points, $m$, determines the accuracy of approximating $\nabla f_\sigma(x)$ by the empirical mean $\frac{1}{m}\sum_{i=1}^m \nabla f(x+\sigma u_i)$. The ablation study results for $m$ are presented in Fig.~\ref{fig:ab_sigma_m_K} (b). It is evident that transferability improves as $m$ increases. Notably, when $m$ reaches 20, the KROCC value drops below 0.4, indicating that larger $m$ effectively reduces noise in the gradients and enhances transferability. However, since a larger $m$ may lead to inefficiency due to the increased sampling complexity (as analyzed in Sec.~\ref{sec:complexity}), we select $m=10$ in this paper. Nevertheless, using a larger $m$, such as 20, could be a reasonable choice if better transferability is needed and some computational cost can be accommodated.

\textbf{The number $K$ of source models in gradient ensembling.} We evaluate how transferability changes as the number of source models increases. Specifically, we add MANIQA~\cite{2022_CVPRw_MANIQA} to the source model set and incrementally include NR-IQA models as source models following the order of HyperIQA, LinearityIQA, LIQE, and MANIQA, calculating the ensemble gradient using Eq.~\eqref{eq:ensemble_smooth_grad}. The results are presented in Fig.~\ref{fig:ab_sigma_m_K} (c). It is evident that incorporating more source models into the ensemble improves transferability. This finding suggests that using as many source models as possible can be beneficial when attacking unknown target models in real-world applications.

\subsection{Computational Complexity and Efficiency}
\label{sec:complexity}
In this subsection, we compare the computational complexity and efficiency of SEGA with other attack methods to demonstrate that SEGA is still an efficient attack method despite utilizing Gaussian sampling and ensembled source models. The computational complexity is assessed based on the number of forward propagations required to generate a single adversarial example. The computational efficiency is measured by the average time to generate an adversarial example. The reported runtime for OUAP includes the time required to generate/train the universal perturbations on the CLIVE test dataset, rather than just the time taken to apply the perturbation to generate an adversarial example. We believe this measurement is appropriate, as training the universal perturbation constitutes the core component of the OUAP method. Experiments are conducted on an NVIDIA GeForce RTX 2080 GPU with 11GB of memory.
The experimental results in Table~\ref{tab:computation} demonstrate that, although SEGA is not the most efficient attack method, it remains effective in generating adversarial examples, requiring only 1.3 seconds per example.

\begin{table}[!t]
    \caption{Comparison of the computational complexity (forward passes per image) and efficiency (generation time per adversarial example) of various attack methods}
    \centering
    \setlength{\tabcolsep}{4pt}
    \resizebox{0.9\linewidth}{!}{
    \begin{tabular}{lcccccc}
    \toprule
        Attack & FGSM & Pattack & OUAP & IOI & Kor & SEGA \\ \midrule
        Forward &  1 & 200 & 10 & 1 & 20 & 30 \\
        Time ($s$) & 0.09 & 9.34 & 0.43 & 0.04 & 18.68 & 1.31 \\ 
    \bottomrule
    \end{tabular}
    }
    \label{tab:computation}
\end{table}

\subsection{Comparison with Query-Based Attacks}
\label{appendix:exp_query}
Transfer-based attacks and query-based attacks are not directly comparable, as each has its own strengths. Transfer-based attacks are more time-efficient, while query-based attacks achieve better performance in black-box scenarios. This distinction has been demonstrated in several studies on attacks against image classification models~\cite{2017_Chen_ZOO_AISec}. Therefore, they are typically not compared directly within a single study on attacks against image classification models~\cite{2022_CompSec_Liu_LowFreTransAdv,2020_NIPS_Chen_AttackFeature}.

Despite this, we still compare our method, SEGA, with the IQA-specific query-based attack proposed by Yang~\etal~\cite{2024_TCSVT_SurFreeIQA} (denoted as SurFree-QA) and Ran~\etal~\cite{2025_ESWA_RandomQuery} to further demonstrate the effectiveness of SEGA. The results against DBCNN are shown in Table~\ref{tab:reuslt_SurFree}. As seen, SEGA outperforms SurFree-QA across all metrics, indicating that SEGA not only achieves better attack performance but is also more time-efficient. Specifically, the R robustness value drops below 0.9 when attacked by SEGA, whereas it remains around 1.4 when attacked by SurFree-QA. In addition to prediction accuracy, the robustness of prediction consistency follows a similar trend. SurFree-QA reduces the SROCC between predicted scores before and after the attack to 0.8397, while SEGA causes a more significant decrease, reducing it to 0.5622. 

In contrast, Ran's method is good at maximizing the absolute prediction error: it attains the highest MAE (41.15) and the lowest R value (0.359), indicating strong distortion of the prediction scores. This phenomenon is reasonable, as their approach iteratively adjusts perturbations to directly optimize the prediction error. However, the proposed method SEGA is most effective at perturbing rank‑based metrics and correlationships—achieving the lowest SROCC (0.562), PLCC (0.626), and KROCC (0.418)—which are the primary indicators of an NR‑IQA model’s ordering ability. 

Moreover, SEGA is very time efficient, requiring only 1.31 seconds to generate one adversarial example, while both query-based methods need more than 28 seconds.

\begin{table}[!t]
   \centering
   \caption{Comparison of SEGA with IQA-specific query-based attacks}
    \setlength{\tabcolsep}{2pt}
    \resizebox{0.98\linewidth}{!}{
    \begin{tabular}{lcccccc}
    \toprule
        ~ & MAE$\uparrow$ & R$\downarrow$ & SROCC$\downarrow$ & PLCC$\downarrow$ & KROCC$\downarrow$ & time(s)$\downarrow$ \\ \midrule
        SurFree-QA & 5.974 & 1.431 & 0.840 & 0.827 & 0.672 & 28.913  \\ 
        Ran~\etal & 41.151 & 0.359 & 0.682 & 0.639 & 0.493 & 29.935 \\
        SEGA & 10.493 & 0.876 & 0.562 & 0.626 & 0.418 & 1.310  \\ 
        \bottomrule
    \end{tabular}
}
\label{tab:reuslt_SurFree}
\end{table}

\section{Conclusion and Discussion}
\label{sec:discuss}

Previous studies have highlighted the difficulty of transferring adversarial examples across NR-IQA models. In this work, we address this challenge by proposing SEGA, a transfer-based black-box attack method. SEGA uses Gaussian smoothing to reduce noise in the source model gradients and ensembles the smoothed gradients of multiple source models to approximate the gradient of the target model. Besides, SEGA incorporates a perturbation filtering module to remove inappropriate perturbations, thereby improving the imperceptibility of adversarial perturbations.

This work takes the first step in exploring transfer-based black-box attacks on NR-IQA models, revealing their vulnerability in black-box scenarios.
It also highlights the critical role of smoothed gradient approximation in improving the transferability of adversarial examples.
This work is limited by the trade-off between transferability, perturbation imperceptibility, and computational efficiency, as discussed in Sec.~\ref{sec:experiment}. Enhancing computational efficiency while maintaining the strong transferability of SEGA will be the focus of our future work. 
Besides, while SEGA is more efficient than query-based attacks, it requires more forward propagations than simpler methods (e.g., FGSM). This higher computational demand may pose a practical limitation for time-sensitive applications. Therefore, enhancing the efficiency of SEGA while preserving its strong transferability will be a key focus of our future work.
 
\section*{Acknowledgments}
This research is partially funded by a grant from the National Natural Science Foundation of China (No. 62506009). We are also grateful to Ran Chen for her contributions to the preliminary experimental work.

\bibliographystyle{IEEEtran}
\bibliography{reference}

\newpage

\section{Biography Section}
 
\vspace{11pt}

\begin{IEEEbiography}
[{\includegraphics[width=1in,height=1.25in,clip,keepaspectratio]{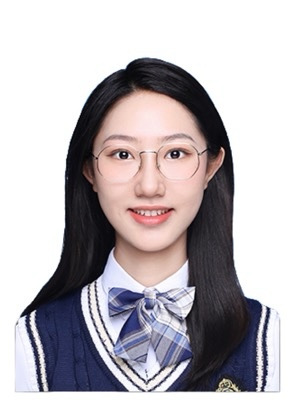}}]
{Yujia Liu}
received the B.S. degree in mathematics from Xiamen University, China, in 2018, and the Ph.D. degree in applied mathematics from Peking University, Beijing, China, in 2023. She is currently a Post-Doctoral Researcher in computer science with Peking University. Her main research interests include computer vision, adversarial attacks, and theoretical studies of adversarial attacks.
\end{IEEEbiography}

\begin{IEEEbiography}
[{\includegraphics[width=1in,height=1.25in,clip,keepaspectratio]{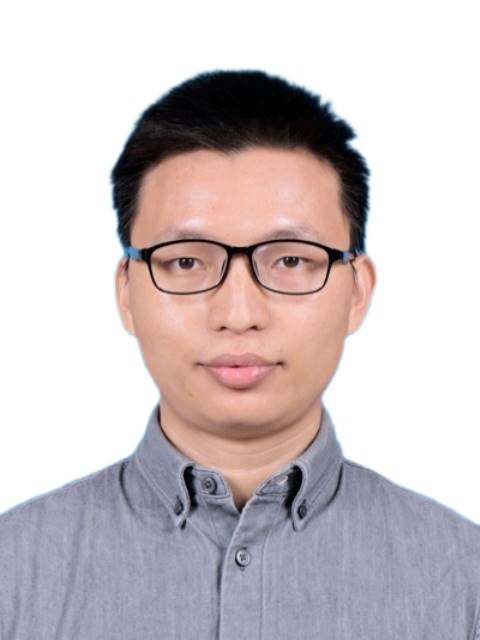}}]
{Dingquan Li}
(Member, IEEE) received dual B.S. degrees in electronic science and technology and applied mathematics from Nankai University, Tianjin, China, in 2015, and his Ph.D. degree in applied mathematics from Peking University, Beijing, China, in 2021. He did his Postdoc and is currently an associate researcher at Pengcheng Laboratory, Shenzhen, China. His research interests include multimedia processing and machine learning, especially quality assessment, data compression, and perceptual optimization. 
\end{IEEEbiography}

\begin{IEEEbiography}
[{\includegraphics[width=1in,height=1.25in,clip,keepaspectratio]{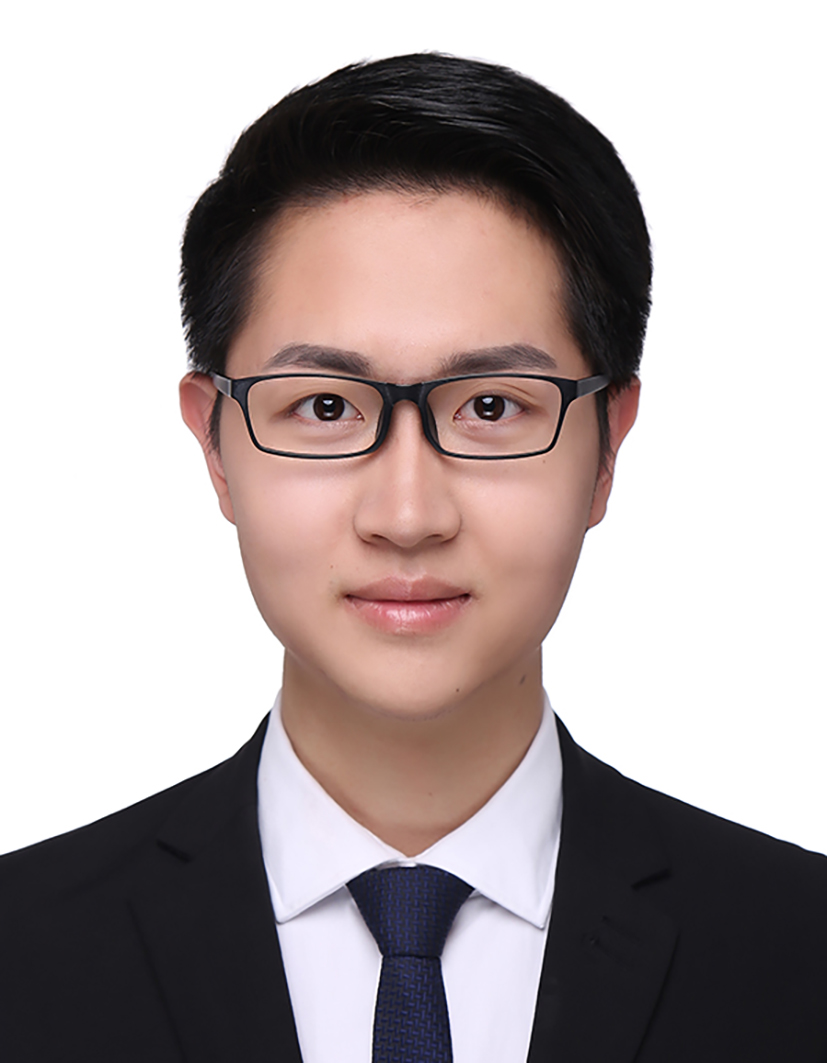}}]
{Zhixuan Li}
received the B.E. degree in computer science from Tianjin University, China, in 2018, and the Ph.D. degree from the National Engineering Research Center of Visual Technology (NERCVT), Peking University, Beijing, China, in 2023. He is currently a Research Fellow with the College of Computing and Data Science, Nanyang Technological University, Singapore. His research interests include computer vision and artificial intelligence, with a focus on complex occlusion scene understanding, vision–language reasoning, and climate forecasting.
\end{IEEEbiography}

\begin{IEEEbiography}
[{\includegraphics[width=1in,height=1.25in,clip,keepaspectratio]{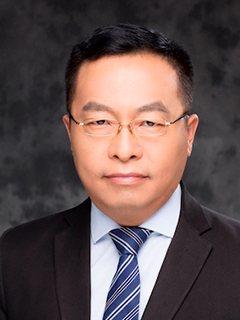}}]
{Tiejun Huang}
(Senior Member, IEEE) is currently a Professor with the School of Computer Science, Peking University, Beijing, China. He has authored or co-authored more than 300 peer-reviewed articles on leading journals and conferences, and he is also a coeditor of four ISO/IEC standards, five National standards, and four IEEE standards. He holds more than 50 granted patents. His research interests include visual information processing and neuromorphic computing. Prof. Huang is a fellow of CAAI and CCF. He was a recipient of the National Award for Science and Technology of China (Tier-2) for three times in 2010, 2012, and 2017. He is the Vice Chair of China National General Group on AI Standardization.
\end{IEEEbiography}

\vfill

\end{document}